\documentclass[runningheads]{llncs}

\usepackage{times}
\usepackage{latexsym}
\usepackage{color}
\usepackage{oldlfont,rawfonts}
\usepackage{graphicx}
\usepackage{verbatim}

\usepackage{amssymb}
\usepackage{amsmath}

\newcommand {\lingconc} {\mathcal{S}}

\newcommand {\ent} {\mathrel{{\scriptstyle\mid\!\sim}}}

\newcommand {\sx} {\langle}
\newcommand {\dx} {\rangle}

\newcommand {\enne} {\mathcal{N}}

\newcommand{\tip}{{\bf T}}

\newcommand{\alc}{\mathcal{ALC}}
\newcommand{\lc}{\mathcal{LC}}

\newcommand{\el}{\mathcal{EL}^{\bot}}

\newcommand{\elpb}{{\mathcal{EL}}^{+}_{\bot}}

\newcommand{\be}{\begin{enumerate}}
\newcommand{\ee}{\end{enumerate}}

\newcommand{\hide}[1]{}

\def \cases{\left \{\begin{array}{l}}
\def \endcases{\end{array}\right .}

\newcommand {\ri} {\rightarrow}
\newcommand {\Ri} {\Rightarrow}

\newcommand {\bes} {\begin{description}}
\newcommand{\ens} {\end{description}}
\newcommand {\la} {\langle}
\newcommand {\ra} {\rangle}

\newcommand {\beq} {\begin{quote}}
\newcommand {\enq} {\end{quote}}
\newcommand {\bit} {\begin{itemize}}
\newcommand {\enit} {\end{itemize}}

\newcommand{\prj}{{\iota}}

\newenvironment{pozz}{\color{black}}{\color{black}}

\begin{document}
\bibliographystyle{plain}

 \title {An ASP approach for reasoning on neural networks 
  under a  finitely  many-valued semantics \\
 for weighted conditional knowledge bases }

\author{Laura Giordano \inst{1} \and Daniele Theseider Dupr{\'{e}}  \inst{1}}

\institute{DISIT - Universit\`a del Piemonte Orientale, 
 Alessandria, Italy 
 \and
Center for Logic, Language and Cognition, 
Dipartimento di Informatica, \\
Universit\`a di Torino, Italy, 
}

\authorrunning{ }
\titlerunning{ }

 \maketitle
 
\begin{abstract}
Weighted knowledge bases for description logics with typicality  have been recently considered under a ``concept-wise'' multipreference semantics (in both the two-valued and fuzzy case), as the basis of a logical semantics of  MultiLayer Perceptrons (MLPs). In this paper we consider weighted conditional $\alc$ knowledge bases with typicality in the finitely many-valued case,  
through three different semantic constructions.
For the boolean fragment $\lc$ of $\alc$ we exploit  ASP and {\em asprin}  for reasoning with the concept-wise multipreference entailment
under a {\em$\varphi$-coherent semantics}, suitable to characterize the stationary states of MLPs. As a proof of concept, we experiment the proposed approach  for checking properties of trained MLPs.

{\em The paper is under consideration for acceptance in TPLP.}

\end{abstract}

\section{Introduction}

Preferential approaches to common sense reasoning \cite{KrausLehmannMagidor:90,Pearl90,whatdoes,BenferhatIJCAI93,Kern-Isberner01}
have been extended to description logics (DLs), to deal with inheritance with exceptions in ontologies,
by allowing for non-strict inclusions,
called {\em typicality or defeasible inclusions},
with different preferential semantics  \cite{lpar2007,sudafricaniKR} 
and closure constructions \cite{casinistraccia2010,CasiniJAIR2013,AIJ15}. 

In recent work, a concept-wise multipreference semantics has been proposed \cite{TPLP2020} as a semantics for ranked DL knowledge bases (KBs),
 i.e. knowledge bases in which defeasible or typicality inclusions of the form $\tip(C) \sqsubseteq D$ (meaning ``the typical $C$'s are $D$'s" or ``normally $C$'s are $D$'s") are given a rank, a natural number, representing 
their strength, where $\tip$ is a typicality operator \cite{lpar2007} that singles out the typical instances of concept $C$. 
The concept-wise multipreference semantics takes into account preferences with respect to different concepts, and integrates them into a single global preference relation,  which is used in the evaluation of defeasible inclusions.
Answer Set Programming (ASP) and, in particular, the {\em asprin}  framework for answer set preferences \cite{BrewkaAAAI15}, is exploited to achieve defeasible reasoning under the multipreference approach for $\elpb$ \cite{rifel}.

In \cite{JELIA2021}, the multi-preferential semantics has been extended to weighted knowledge bases, in which typicality inclusions have a real (positive or negative) weight, representing plausibility or implausibility. 
The multipreference semantics has been exploited to provide a preferential interpretation to Multilayer Perceptrons (MLPs) \cite{Haykin99}, an approach previously considered \cite{CILC2020,JLC2022} for self-organising maps  (SOMs) \cite{kohonen2001}.
In both cases, considering the domain of all input stimuli presented to the network during training (or in the generalization phase), one can build a semantic interpretation describing the input-output behavior of the network as a multi-preference interpretation, where preferences are associated to concepts.
For MLPs, based on the fuzzy multipreference semantics for weighted KBs, 
a deep neural network can actually be regarded as a weighted conditional knowledge base \cite{JELIA2021}. 
This rises the issue of defining proof methods for reasoning with weighted conditional knowledge bases.

Undecidability results for fuzzy DLs with general inclusion axioms 
\cite{CeramiStraccia2011,BorgwardtPenaloza12} 
motivate the investigation of many-valued approximations of fuzzy multipreference entailment.
In this paper,  we restrict to the case of  finitely many-valued Description Logics 
 \cite{GarciaCerdanaAE2010,BobilloStraccia_InfSci_11,BobilloDelgadoStraccia2012,BorgwardtPenaloza13},
and reconsider the fuzzy multipreference semantics based on the notions of {\em coherent}  \cite{JELIA2021}, {\em faithful} \cite{ECSQARU2021} and   {\em $\varphi$-coherent} \cite{TR_Argumentation} model of a defeasible KB.
The last notion is suitable to characterize the stationary states of MLPs and is related to the previously introduced notion of coherent multipreference interpretation.

We consider the finitely many-valued G\"odel description logic $G_n \alc$, and the  finitely many-valued \L ukasiewicz DL, $\L _n \alc$, and develop their extension with typicality and a semantic closure construction based on coherent, faithful and $\varphi_n$-coherent interpretations to deal with weighted KBs.
For the boolean fragment $\lc$ of $\alc$, which neither contains roles, nor universal and existential restrictions,
we develop an ASP approach for deciding $\varphi_n$-coherent entailment from weighted knowledge bases in the finitely many-valued case.
In particular, we  develop an ASP encoding of a weighted KB and exploit  {\em asprin} \cite{BrewkaAAAI15} for defeasible reasoning, to prove typicality properties of a weighted conditional KB. From the soundness and completeness of the encoding, we also get a $\Pi^p_2$ complexity upper-bound for $\varphi_n$-coherent entailment.

As a proof of concept, we experiment our approach over weighted KBs corresponding to some of the trained multilayer feedforward networks
considered by Thrun et al. \cite{monk}. We exploit ASP to verify some properties of the network expressed as typicality properties in the finite many-valued case.
This is a step towards explainability of the black-box, in view of a trustworthy, reliable and explainable AI \cite{Adadi18,Guidotti2019,Arrieta2020}, and 
of an integrated use of symbolic reasoning and neural models.

\section{Finitely many-valued $\alc$}  \label{sec:ALC}

Fuzzy description logics have been widely studied in the literature for representing vagueness in DLs \cite{Straccia05,Stoilos05,LukasiewiczStraccia09,BorgwardtPenaloza12,BobilloOWL2EL2018},  
based on the idea that concepts and roles can be interpreted 
as fuzzy sets and fuzzy relations.

In fuzzy logic formulas have a truth degree from a truth space  $\cal S$, usually $[0, 1]$ (as in in Mathematical Fuzzy Logic \cite{Cintula2011})
or $\{0, \frac{1}{n},\ldots, \frac{n-1}{n}, \frac{n}{n}\}$, for an integer $n \geq 1$. $\cal S$ may as well be a complete lattice or a bilattice.

The finitely many-valued case is also well studied for DLs \cite{GarciaCerdanaAE2010,BobilloStraccia_InfSci_11,BobilloDelgadoStraccia2012,BorgwardtPenaloza13} and, in the following, we will consider a finitely many-valued extension of $\alc$ with typicality.

The basic $\alc$ syntax features a set ${N_C}$ of concept names,  a set ${N_R}$ of role names and a set ${N_I}$ of individual names.  
The set  of $\alc$ \emph{concepts} can be
defined inductively: \\ 
- $A \in N_C$, $\top$ and $\bot$ are {concepts};\\
- if $C$ and $ D$ are concepts, and $r \in N_R$, then $C \sqcap D,\; C \sqcup D,\; \neg C, \; \forall r.C,\; \exists r.C$ 
are {concepts}.

We assume the truth space to be 
${\cal C}_n= \{0, \frac{1}{n},\ldots, \frac{n-1}{n}, \frac{n}{n}\}$, for an integer $n \geq 1$.
A {\em finitely many-valued interpretation} for $\alc$ is a pair $I=\langle \Delta, \cdot^I \rangle$ where:
$\Delta$ is a non-empty domain and 
$\cdot^I$ is an {\em interpretation function} that assigns 
to each $a \in N_I$ a value $a^I \in\Delta$, 
to each  $A\in N_C$ a function  $A^I :  \Delta \ri {\cal C}_n $,
to each  $r \in N_R$  a function  $r^I:   \Delta \times  \Delta \ri  {\cal C}_n $. 
A domain element $x \in \Delta$ 
belongs to the extension of concept name $A$ to some degree $A^I(x)$ in ${\cal C}_n$.

The  interpretation function $\cdot^I$ is extended to complex concepts as follows: 

$\mbox{\ \ \ }$ $\top^I(x)=1$, $\mbox{\ \ \ \ \ \ \ \ \ \  \ \ \ \ \ \ \ \ \ \ }$ $\bot^I(x)=0$,  $\mbox{\ \ \ \ \ \ \ \ \ \ \ \ \ \ \  \ \ \ \ \ }$  $(\neg C)^I(x)= \ominus C^I(x)$, 

$\mbox{\ \ \ }$  $(\exists r.C)^I(x) = sup_{y \in \Delta} \; r^I(x,y) \otimes C^I(y)$,  $\mbox{\ \ \ \ \ \ }$  $(C \sqcup D)^I(x) =C^I(x) \oplus D^I(x)$ 

$\mbox{\ \ \ }$  $(\forall r.C)^I (x) = inf_{y \in \Delta} \;  r^I(x,y) \rhd C^I(y)$, $\mbox{\ \ \ \ \ \ \ }$  $(C \sqcap D)^I(x) =C^I(x) \otimes D^I(x)$ 

\noindent
where  $x \in \Delta$ and $\otimes$, $\oplus$, $\rhd$ and $\ominus$ are arbitrary but fixed t-norm, s-norm, implication function, and negation function \cite{LukasiewiczStraccia09}.
In particular, in this paper we consider two finitely many-valued description logics based on $\alc$, 
the finitely many-valued \L ukasiewicz description logic $\alc$ (called $\L_n \alc$ in the following) as well as 
the finitely many-valued G\"odel description logic $\alc$, 
extended with a standard involutive negation $ \ominus a = 1-a$ (called $G_n \alc$ in the following).
Such logics are defined along the lines of the finitely many-valued description logic $\cal SROIQ$ \cite{BobilloStraccia_InfSci_11}, the logic GZ $\cal SROIQ$ \cite{BobilloDelgadoStraccia2012},
and the logic  $\alc^*(S)$ \cite{GarciaCerdanaAE2010}, where $*$ is a divisible finite t-norm over a chain of n elements. 

Specifically, in an $\L_n \alc$ interpretation, we let:  $a \otimes b= max\{a+b-1,0\}$,  $a \oplus b= min\{a+b,1\}$, 
 $a \rhd b= min\{1- a+b,1\}$ and $ \ominus a = 1-a$. 
In a $G_n \alc$ interpretation, we let: $a \otimes b= min\{a,b\}$,  $a \oplus b= max\{a,b\}$,  $a \rhd b= 1$ {\em if} $a \leq b$ {\em and} $b$ {\em otherwise}; and $ \ominus a = 1-a$.

The  interpretation function $\cdot^I$ is also extended  to $\alc$ concept inclusions of the form $C \sqsubseteq D$ (where $C$ and $D$ are $\alc$ concepts), and to $\alc$ assertions of the form $C(a)$ and $r(a,b)$  (where $C$ is an $\alc$ concept, $r\in N_R$, $a,b \in N_I$),  
as follows:\\
 $(C \sqsubseteq D)^I= inf_{x \in \Delta}  C^I(x) \rhd D^I(x)$,
$\mbox{\ \ \ }$  $(C(a))^I=C^I(a^I)$,  $\mbox{\ \ \ }$  $(R(a,b))^I=R^I(a^I,b^I)$.

A {\em  $G_n \alc$ ($\L_n \alc$)  knowledge base} 
$K$ is a pair $({\cal T}, {\cal A})$ where ${\cal T}$ is a TBox  and ${\cal A}$ an ABox. A TBox ${\cal T}$ is a set of $G_n \alc$ ($\L_n \alc$)  {\em concept inclusions} of the form $C \sqsubseteq D \;\theta\; \alpha$, where $C \sqsubseteq D$ is an $\alc$ concept inclusion, $\theta \in \{\geq,\leq,>,<\}$ and $\alpha \in [0,1]$. An ABox ${\cal A}$ is a set of $G_n \alc$ ($\L_n \alc$)  {\em assertions} of the form $C(a) \; \theta \alpha$ \ or \ $r(a,b) \; \theta \alpha$, where $C$ is an $\alc$ concept, $r\in N_R$, $a,b \in N_I$,  $\theta \in \{{\geq,}\leq,>,<\}$ and $\alpha \in [0,1]$.

The notions of satisfiability of a KB  in a many-valued interpretation and of $G_n \alc$ ($\L_n \alc$) entailment are defined as follows: 
\begin{definition}[Satisfiability and entailment for $G_n \alc$ and $\L_n \alc$] \label{satisfiability}
A  $G_n \alc$ ($\L_n \alc$)  interpretation $I$ {\em satisfies} a $G_n \alc$ ($\L_n \alc$)  axiom $E$, 
as follows:

- $I$ satisfies 
axiom $C \sqsubseteq D \;\theta\; \alpha$ if $(C \sqsubseteq D)^I \theta\; \alpha$;

- $I$ satisfies assertion $C(a) \; \theta \; \alpha$ if $C^I(a^I) \theta\; \alpha$;
 
- $I$ satisfies assertion $r(a,b) \; \theta \; \alpha$ if $r^I(a^I,b^I) \theta\; \alpha$.

\vspace{1mm}

\noindent
Given  a $G_n \alc$ ($\L_n \alc$) knowledge base $K=({\cal T}, {\cal A})$,
a $G_n \alc$ ($\L_n \alc$)  interpretation $I$  satisfies ${\cal T}$ (resp. ${\cal A}$) if $I$ satisfies all  inclusions in ${\cal T}$ (resp. all assertions in ${\cal A}$).
A $G_n \alc$ ($\L_n \alc$) interpretation $I$ is a $G_n \alc$ ($\L_n \alc$) \emph{model} of $K$ if $I$ satisfies ${\cal T}$ and ${\cal A}$.
A $G_n \alc$ ($\L_n \alc$) axiom $E$   {\em is entailed by  knowledge base $K$}, written $K \models_{G_n \alc} E$ (resp. $K \models_{\L_n \alc} E$), if for all $G_n \alc$ ($\L_n \alc$) models $I=$$\sx \Delta,  \cdot^I\dx$ of $K$, 
$I$ satisfies $E$.
\end{definition}

\section{Finitely many-valued $\alc$ with typicality  }\label{sec:fuzzyalc+T} 

In this section, we consider an extension of finitely many-valued $\alc$ with typicality concepts, based on a preferential semantics, first introduced by Giordano and Theseider Dupr\'e \cite{JELIA2021} for weighted $\el$ knowledge bases (we adopt an equivalent slight reformulation of the semantics developed for fuzzy $\alc$ 
\cite{ECSQARU2021}).
%
%
%
The idea is similar to the extension of $\alc$ with typicality in the two-valued case \cite{lpar2007} 
but 
the degree of membership of domain individuals in a concept $C$ is used to identify the typical elements of $C$.
The extension allows for the definition of {\em  typicality inclusions} of the form 
$\tip(C) \sqsubseteq D \;\theta \; \alpha$. 
For instance, $\tip(C) \sqsubseteq D \geq  \alpha$
means that typical $C$-elements are $D$-elements with degree greater than $\alpha$. In the two-valued case, a typicality  inclusion $\tip(C) \sqsubseteq D$ corresponds to a KLM conditional implication $C \ent D$ \cite{KrausLehmannMagidor:90,whatdoes}. 
%
As in the two-valued case,  
 nesting of the typicality operator is not allowed.

Observe that, in a many-valued $\alc$ interpretation $I= \langle \Delta, \cdot^I \rangle$, the degree of membership $C^I(x)$ of the domain elements $x$ in a concept $C$ induces a preference relation $<_C$ on $\Delta$: 
\begin{equation}\label{def:induced_order}
x <_C y \mbox{ iff } C^I(x) > C^I(y)
\end{equation}
For a finitely many-valued $\alc$ interpretation $I= \langle \Delta, \cdot^I \rangle$, each preference relation $<_{C}$ has the properties of preference relations in KLM-style ranked interpretations \cite{whatdoes}, that is,  $<_{C}$ is a modular and well-founded strict partial order. 
Let us recall that $<_{C}$ is {\em well-founded} 
if there is no infinite descending chain 
of domain elements;
    $<_{C}$ is {\em modular} if,
for all $x,y,z \in \Delta$, $x <_{C} y$ implies ($x <_{C} z$ or $z <_{C} y$).
Well-foundedness holds for the induced preference $<_C$ defined by condition (\ref{def:induced_order}) as we have assumed the truth space to be ${\cal C}_n$.
We will denote the extensions of $\L_n \alc$ and $G_n \alc$ with typicality, respectively, by $\L_n \alc \tip$ and $G_n \alc \tip$.

Each relation $<_C$ has  the properties of a preference relation in KLM  rational interpretations \cite{whatdoes}, also called ranked interpretations. As  many-valued interpretations induce multiple preferences, they can be regarded as {\em multi-preferential} interpretations, which have also been studied in the two-valued case \cite{TPLP2020,Delgrande2020,AIJ21,Casini21}.

The preference relation $<_C$ captures the relative typicality of domain elements wrt concept $C$ and may then be used to identify the {\em typical  $C$-elements}. We regard typical $C$-elements as the domain elements $x$ that  are preferred with respect to $<_C$
among the ones such that $C^I(x) \neq 0$. 
Let $C^I_{>0}$ be the crisp set containing all domain elements $x$ such that $C^I(x)>0$, that is, $C^I_{>0}= \{x \in \Delta \mid C^I(x)>0 \}$.
One can provide a (two-valued) interpretation of typicality concepts $\tip(C)$ with respect to an interpretation $I$ as: 
\begin{align}\label{eq:interpr_typicality}
	(\tip(C))^I(x)  & = \left\{\begin{array}{ll}
						 1 & \mbox{ \ \ \ \  if } x \in min_{<_C} (C^I_{>0}) \\
						0 &  \mbox{ \ \ \ \  otherwise } 
					\end{array}\right.
\end{align} 
where $min_<(S)= \{u: u \in S$ and $\nexists z \in S$ s.t. $z < u \}$.  When $(\tip(C))^I(x)=1$, $x$ is said to be a typical $C$-element in $I$.
Note that, if $C^I(x)>0$ for some $x \in \Delta$,  
$min_{<_C} (C^I_{>0}) \neq \emptyset$. 
This generalizes the property that, in the crisp case, $C^I\neq \emptyset$ implies  $(\tip(C))^I\neq \emptyset$.
\begin{definition}[$G_n \alc \tip$ ($\L_n \alc \tip$) interpretation]
A $G_n \alc \tip$ (resp., $\L_n \alc \tip$) interpretation $I= \langle \Delta, \cdot^I \rangle$ is a finitely many-valued $G_n \alc$ (resp., $\L_n \alc \tip$) interpretation over ${\cal C}_n$, extended by interpreting typicality concepts according to (\ref{eq:interpr_typicality}).
\end{definition}
The many-valued interpretation  $I= \langle \Delta, \cdot^I \rangle$ implicitly defines a multi-preferential interpretation, where any concept $C$ is associated to a preference  relation $<_C$.  
The notions of {\em satisfiability} in $G_n \alc \tip$ ($\L_n \alc \tip$), 
of {\em model} of a $G_n \alc \tip$ ($\L_n \alc \tip$) knowledge base, and of $G_n \alc \tip$ ($\L_n \alc \tip$)  {\em entailment} can be defined 
similarly to those for $\L_n \alc$ and $G_n \alc$ in Section  \ref{sec:ALC}.

\subsection{Weighted KBs and closure construction for finitely many values} \label{sec:closure}

In this section we introduce the notion of weighted $G_n \alc \tip$ ($\L_n \alc \tip$) knowledge base allowing for {\em weighted defeasible inclusions}, namely, typicality inclusions with a real-valued weight, as introduced for  $\cal EL$ in \cite{JELIA2021}. 

A  {\em weighted $G_n \alc \tip$ knowledge base} $K$, over a set ${\cal C}= \{C_1, \ldots, C_k\}$ of distinguished $G_n \alc$ concepts,
is a tuple $\langle  {\cal T}, {\cal T}_{C_1}, \ldots, {\cal T}_{C_k}, {\cal A}  \rangle$, where  ${\cal T}$  is a set of $G_n \alc$ inclusion axioms, 
${\cal A}$ is a set of $G_n \alc$ assertions  
and
${\cal T}_{C_i}=\{(d^i_h,w^i_h)\}$ is a set of all weighted typicality inclusions $d^i_h= \tip(C_i) \sqsubseteq D_{i,h}$ for $C_i$, indexed by $h$, where each inclusion $d^i_h$ has weight $w^i_h$, a real number, and $C_i$ and $D_{i,h}$ are $G_n \alc$ concepts.
As in \cite{JELIA2021}, the typicality operator is assumed to occur only on the left hand side of a weighted typicality inclusion, and we call {\em distinguished concepts}  those concepts $C_i$ occurring on the l.h.s. of some typicality inclusion $\tip(C_i) \sqsubseteq D$.
The definition of a weighted $\L_n \alc \tip$ knowledge base is similar.
Let us consider the following example.

\begin{example} \label{exa:Penguin} 
Consider the weighted  $G_n \alc \tip$  knowledge base $K =\langle {\cal T},  {\cal T}_{Bird}, {\cal T}_{Penguin},$  $ {\cal A} \rangle$, over the set of distinguished concepts ${\cal C}=\{\mathit{Bird, Penguin}\}$, with 
 $ {\cal T}$ containing, for instance, the inclusion $\mathit{Black \sqcap Red  \sqsubseteq  \bot \geq 1}$.

\noindent
The weighted TBox ${\cal T}_{Bird} $ 
contains the weighted defeasible inclusions: 

$(d_1)$ $\mathit{\tip(Bird) \sqsubseteq Fly}$, \ \  +20  \ \ \ \ \ \ \ \ \  \ \ \ \ \   $(d_2)$ $\mathit{\tip(Bird) \sqsubseteq  Has\_Wings}$, \ \ +50

$(d_3)$ $\mathit{\tip(Bird) \sqsubseteq   Has\_Feather}$, \ \ +50.

\noindent
and ${\cal T}_{Penguin}$ contains the weighted defeasible inclusions:

$(d_4)$ $\mathit{\tip(Penguin) \sqsubseteq Bird}$, \ \ +100 \ \ \ \ \ \ \ \ \ \ \ $(d_5)$ $\mathit{\tip(Penguin) \sqsubseteq  Fly}$, \ \ - 70  

$(d_6)$ $\mathit{\tip(Penguin) \sqsubseteq Black}$, \ \  +50.

\noindent
I.e., a bird normally has wings, has feathers and flies, but having wings and feather (both with weight 50)  for a bird is more plausible than flying (weight 20), although flying is regarded as being plausible; and so on. 
Given Abox ${\cal A}$ in which Reddy is red, has wings, has feather and flies (all with degree 1) and Opus has wings and feather (with degree 1), is black with degree 0.8 and does not fly, 
considering the weights of defeasible inclusions, we 
expect Reddy to be more typical than Opus as a bird, but less typical 
as a penguin. 
\end{example}
\normalcolor

In previous work \cite{JELIA2021} a semantics of a weighted ${\cal EL}$  knowledge bases has been defined through a {\em semantic closure construction}, similar in spirit to Lehmann and Magidor's rational closure \cite{whatdoes}, Lehmann's lexicographic closure \cite{Lehmann95}, and related to c-representations \cite{Kern-Isberner01}, but based on multiple preferences. 
Here, we extend the same construction to weighted $G_n \alc \tip$ ($\L_n \alc \tip$) knowledge bases, 
by considering the notions of  {\em coherent},  {\em faithful} and {\em $\varphi$-coherent} interpretations.
The construction allows a subset of the  $G_n \alc \tip$ ($\L_n \alc \tip$) interpretations to be selected, 
those 
in which the preference relations $<_{C_i}$ {\em faithfully} represent the defeasible part of the knowledge base $K$.

Let ${\cal T}_{C_i}=\{(d^i_h,w^i_h)\}$ be the set of weighted typicality inclusions $d^i_h= \tip(C_i) \sqsubseteq D_{i,h}$ associated to the distinguished concept $C_i$, and let $I=\langle \Delta, \cdot^I \rangle$ be a $G_n \alc \tip$ ($\L_n \alc \tip$)  interpretation.
In the two-valued case, we would associate to each domain element $x \in \Delta$ and each distinguished concept $C_i$, a weight $W_i(x)$ of $x$ wrt $C_i$ in $I$, by {\em summing the weights} of the defeasible inclusions satisfied by $x$.
However, as $I$ is a many-valued interpretation, 
 we  need to consider, for all inclusions $\tip(C_i) \sqsubseteq D_{i,h} \in {\cal T}_{C_i}$,  
the degree of membership of $x$ in $D_{i,h}$. 
For each domain element $x \in \Delta$ and distinguished concept $C_i$, {\em the weight $W_i(x)$ of $x$ wrt $C_i$} in a $G_n \alc \tip$ ($\L_n \alc \tip$)  interpretation $I=\langle \Delta, \cdot^I \rangle$ is:  
 \begin{align}\label{weight_fuzzy}
	W_i(x)  & = \left\{\begin{array}{ll}
						 \sum_{h} w_h^i  \; D_{i,h}^I(x) & \mbox{ \ \ \ \  if } C_i^I(x)>0 \\
						- \infty &  \mbox{ \ \ \ \  otherwise }  
					\end{array}\right.
\end{align} 
where $-\infty$ is added at the bottom of ${\mathbb{R}}$.
The value of $W_i(x) $ is $- \infty $ when $x$ is not a $C$-element (i.e., $C_i^I(x)=0$). 
Otherwise, $C_i^I(x) >0$ and the higher is the sum $W_i(x) $, the more typical is the element $x$ relative to the defeasible properties of $C_i$.

\begin{example} \label{exa:penguin2}
Let us consider again Example \ref{exa:Penguin}.
Let $I$ be an $G_n \alc \tip$ interpretation such that $\mathit{Fly^I(red}$- $\mathit{dy)}$ $\mathit{= (Has\_Wings)^I (reddy)= (Has\_Feather)^I (reddy)=1}$ and  
\linebreak  $\mathit{Red^I(reddy) =1 }$, 
and 
$\mathit{Black}^I (reddy)=0$. Suppose further that $\mathit{Fly^I(opus) = 0}$ and $\mathit{ (Has\_Wings)^I (opus)= } $ $ \mathit{=(Has\_Feather)^I }$ $\mathit{(opus)=1 }$ and $\mathit{ Black^I(opus) =0.8}$. 
Considering the weights of typicality inclusions for $\mathit{Bird}$,  $\mathit{W_{Bird}(reddy)= 20+50+}$ $\mathit{50=120}$ and $\mathit{W_{Bird}(opus) =0+50+50=100}$.
This suggests that Reddy should be more typical as a bird than Opus.
On the other hand, if we suppose that $\mathit{Bird^I(reddy)}$ $=1$ and $\mathit{Bird^I(opus)=0.8}$, then $\mathit{W_{Penguin}}$ $\mathit{(reddy)}$ $ \mathit{= 100-70=30}$ and $\mathit{W_{Penguin}}$ $\mathit{(opus)= }$ $\mathit{0.8 \times 100+0.8 \times 50}$ $\mathit{=120}$, 
and Reddy should be less typical as a penguin than Opus.
 \end{example}

In previous work \cite{JELIA2021} a notion of {\em coherence} is introduced, to force an agreement between the preference relations  $<_{C_i}$ induced by a fuzzy interpretation $I$, for distinguished concepts $C_i$, and the weights $W_i(x)$ computed, for each $x \in \Delta$, from the knowledge base $K$, given the interpretation $I$. 
In the many-valued case, this leads to the following definition of coherent multipreference model of a weighted  $G_n \alc \tip$ ($\L_n \alc \tip$) knowledge base.

\begin{definition}[Coherent multipreference model of a weighted $G_n \alc \tip$/$\L_n \alc \tip$ KB]\label{fuzzy_cfm-model} 
Let $K=\langle  {\cal T},$ $ {\cal T}_{C_1}, \ldots,$ $ {\cal T}_{C_k}, {\cal A}  \rangle$ be  a weighted $G_n \alc \tip$ ($\L_n \alc \tip$)  knowledge base  over  ${\cal C }$. 
A {\em  coherent multipreference model (cm-model)}  of $K$ is  a $G_n \alc \tip$ ($\L_n \alc \tip$)  interpretation $I=\langle \Delta, \cdot^I \rangle$  
s.t.: 
\begin{itemize}
\item
$I$  satisfies  the inclusions in $ {\cal T}$ and the assertions in ${\cal A}$;
\item 
for all $C_i\in {\cal C}$,  the preference {\em $<_{C_i}$   is coherent  to $ {\cal T}_{C_i}$}, that is, for all $x,y \in \Delta$,
\begin{align}\label{coherence_2}
x  <_{C_i}  y & \iff W_i(x) > W_i(y)  
\end{align}
\end{itemize}
\end{definition}
In a similar way, one can define a {\em faithful multipreference model (fm-model) of $K$} by replacing   the {\em coherence} condition (\ref{coherence_2}) with a {\em faithfulness} condition: 
for all $x,y \in \Delta$, 
\begin{align} \label{faitfulness}
x  <_{C_i}  y & \Ri W_i(x) > W_i(y) .
\end{align}
The weaker notion of faithfulness allows to define a larger class of multipreference models of a weighted knowledge base, compared to the class of coherent models. 
This allows a larger class of monotone non-decreasing activation functions in neural network models to be captured, whose activation function is monotonically non-decreasing (we refer to  the work by Giordano and Theseider Dupr\'e  \cite{JELIA2021}, 
and by Giordano \cite{ECSQARU2021}.

\section{$\varphi$-coherent models with finitely many values} \label{sec:varphi_coherent_models}

In this section we consider another notion of coherence of a many-valued interpretation $I$ wrt a KB, that we call {\em $\varphi$-coherence}, where $\varphi$ 
is a function  from $\mathbb{R}$ to the interval $[0,1]$, i.e., $\varphi: {\mathbb{R}} \rightarrow [0,1]$. $\varphi$-coherent models have been first introduced  
in the definition of a gradual argumentation semantics \cite{TR_Argumentation}.
Let us consider $\varphi$-coherent $G_n \alc \tip$ ($\L_n \alc \tip$) interpretations.
 
 \begin{definition}[$\varphi$-coherence]\label{varphi-coherence} 
 Let $K=\langle  {\cal T},$ $ {\cal T}_{C_1}, \ldots,$ $ {\cal T}_{C_k}, {\cal A}  \rangle$ be  a weighted $G_n \alc \tip$ ($\L_n \alc \tip$) knowledge base, and $\varphi: {\mathbb{R} } \rightarrow [0,1]$.
A $G_n \alc \tip$ ($\L_n \alc \tip$) interpretation $I=\langle \Delta, \cdot^I \rangle$  is {\em $\varphi$-coherent} if, 
for all concepts $C_i \in {\cal C}$ and $x\in \Delta$, 
\begin{align}\label{fi_coherence}
C_i^I(x)= \varphi  (\sum_{h} w_h^i  \; D_{i,h}^I(x)) 
\end{align}
where ${\cal T}_{C_i}=\{(\tip(C_i) \sqsubseteq D_{i,h},w^i_h)\}$ is the set of weighted conditionals  for $C_i$.
A {\em $\varphi$-coherent multipreference model ($\varphi $-coherent model)} of a knowledge base $K$, is defined as a coherent model in Definition \ref{fuzzy_cfm-model}, by replacing the notion of {\em coherence} in condition (\ref{coherence_2}) with the notion of {\em $\varphi$-coherence}  (\ref{fi_coherence}).

\end{definition}

The relationships between the three semantics \cite{TR_Argumentation} extend to the finite many-valued case as follows.

\begin{proposition} \label{prop:phi_coherent_models}
Let $K$ be a weighted $G_n \alc \tip$ ($\L_n \alc \tip$)  knowledge base and $\varphi: {\mathbb{R}} \rightarrow [0,1]$.
(1) if $\varphi$ is a {\em monotonically non-decreasing} function, a $\varphi$-coherent  multipreference model $I$ of $K$ is also a faithful-model of $K$;
(2) if $\varphi$ is a {\em monotonically increasing} function, a $\varphi$-coherent multipreference model $I$ of $K$ is also a coherent-model of $K$. 
\end{proposition}

To see that the set of equations defined by (6)  allow to characterize the {\em stationary states} of Multilayer Perceptrons (MLPs), let us consider from  \cite{Haykin99} the model of a {\em neuron} as an information-processing unit in an (artificial) neural network. The basic elements are the following:
(1) a set of {\em synapses} or {\em connecting links}, each one characterized by a {\em weight}. 
We let $x_j$ be the signal at the input of synapse $j$ connected to neuron $i$, and $w_{ij}$ the related synaptic weight;
(2) the adder for summing the input signals to the neuron, weighted by the respective synapses weights: $\sum^n_{j=1} w_{ij} x_j$;
(3) an {\em activation function} for limiting the amplitude of the output of the neuron (here, we assume, to the interval $[0,1]$).
A neuron $i$ can be described by the following pair of equations: 
$u_i= \sum^n_{j=1} w_{ij} x_j $
and $y_i=\varphi(u_i + b_i)$
where  $x_1, \ldots, x_n$ are the input signals and $w_{i1}, \ldots,$ $ w_{in} $ are the weights of neuron $i$; 
$b_i$ is the bias, $\varphi$ the activation function, and $y_i$ is the output signal of neuron $i$.
By adding a new synapse with input $x_0=+1$ and synaptic weight $w_{i0}=b_i$, one can write: 
$u_i= \sum^n_{j=0} w_{ij} x_j $, and  $y_i=\varphi(u_i)$,
where $u_i$ is called the {\em induced local field} of the neuron.

A neural network $\enne$ can then be seen as ``a directed graph consisting of nodes with interconnecting synaptic and activation links"  \cite{Haykin99}.
Nodes in the graph are the neurons (the processing units)  
and the weight $w_{ij}$ on the edge from node $j$ to node $i$ represents the strength of the connection between unit $j$ and unit $i$.

A mapping of a neural network to a conditional KB can be defined in a simple way  \cite{JELIA2021}, by associating a concept name $C_i$ with each unit $i$ in the network and by introducing, for
each synaptic connection from neuron $h$ to neuron $i$ with weight $w_{ih}$, a conditional $\tip(C_i) \sqsubseteq C_h$ with weight $w_h^i=w_{ih}$.
 If we assume that $\varphi$ is the {\em activation function} of {\em all units} in the network $\enne$ and we consider the  infinite-valued fuzzy logic with truth space ${\cal S} = [0,1]$,  then the solutions of equations (\ref{fi_coherence}) characterize the {\em stationary states} of MLPs, where 
$C_i^I(x)$ corresponds to the activation of neuron $i$ for some input stimulus $x$, each $D_{i,h}^I(x)$ corresponds to the input signal $x_h$, 
and $\sum_{h} w_h^i  \; D_{i,h}^I(x)$ corresponds to the {\em induced local field} of neuron $i$.

Notice that, when the truth space is the finite set ${\cal C}_n$, for $n\geq 1$, the notion of $\varphi$-coherence may fail to characterize all the  stationary states of a network, simply as there may be stationary states such that the activity values of units fall outside ${\cal C}_n$.
In the next section, we will consider an approximation $\varphi_n$ of the function $\varphi$ over ${\cal C}_n$, with the idea to capture an approximated 
behavior of the network  based on the finite many-valued semantics of a weighted conditional KB, and to construct a preferential model for properties verification.

For a weighted $G_n \alc \tip$ ($\L_n \alc \tip$) knowledge base $K$, a notion of {\em coherent/ faithful/$\varphi$-coherent entailment} 
can be defined in a natural  way. 
As in the two-valued case \cite{AIJ15}, we restrict our consideration 
 to {\em canonical models}, i.e., models which are large enough to contain 
all the relevant domain elements with their different valuations.
Informally, a canonical $\varphi$-coherent model of $K$ is a $\varphi$-coherent model of $K$ that contains a domain element for each possible valuation of concepts which is present in any $\varphi$-coherent model of $K$. 
Similarly for coherent and faithful models.

\begin{definition}[Canonical coherent/faithful/$\varphi$-coherent model of $K$]
Given a  weighted $G_n \lc \tip$ ($\L _n \lc \tip$) knowledge base $K$, 
$I=(\Delta, \cdot^I)$ is a canonical  coherent/faithful/$\varphi$-coherent model of $K$  if: (i) $I$ is  a coherent/faithful/$\varphi$-coherent model of $K$  and,
(ii) for each coherent/faithful/$\varphi$-coherent model $J=(\Delta^J, \cdot^J)$ of $K$ and each $y \in \Delta^J$,
there is an element $z \in \Delta$ such that $B^I(z)=B^J(y)$, for all concept names $B$ occurring in $K$.
\end{definition}
A result concerning the existence of canonical $\varphi$-coherent  models, for weighted KBs having at least a $\varphi$-coherent  model,  can be found in the supplementary
material for the paper, Appendix A. Let us define entailment.

\begin{definition}[coherent/faithful/$\varphi$-coherent entailment] \label{fm-entailment}
Given a weighted $G_n \alc \tip$ ($\L_n \alc \tip$)  knowledge base $K$,
a $G_n \alc \tip$ ($\L_n \alc \tip$) axiom $E$   is {\em coherently/faithfully/$\varphi$-coherently entailed} from $K$  if, for all  {\em canonical coherent/ faithful/$\varphi$-coherent models} $I=\langle \Delta, \cdot^I \rangle$ of $K$, $I$ satisfies $E$.
\end{definition}

 The properties of faithful entailment in the fuzzy case have been studied by Giordano \cite{ECSQARU2021}.
Faithful entailment is well-behaved: it deals with specificity and irrelevance; it is not subject to inheritance blocking;
it satisfies most KLM properties of a preferential consequence relation \cite{KrausLehmannMagidor:90,whatdoes}, depending on their fuzzy reformulation and  on the chosen combination functions.

In the next section, we restrict our consideration to the boolean fragment $\lc$ of $\alc$ (with neither roles, nor universal nor existential restrictions), which is sufficient to encode MLPs as weighted KBs and to formulate boolean properties of the network. 
We consider the finitely many-valued logics $G_n \lc \tip$ and $\L_n \lc \tip$,
and exploit ASP and {\em asprin} for defeasible reasoning in $G_n \lc \tip$ and $\L_n \lc \tip$ under an approximation $\varphi_n$ of $\varphi$.

\section{ASP and {\em asprin} for 
	reasoning in  $G_n \lc \tip$ and $\L_n \lc \tip$: $\varphi_n$-coherence and  verification of multi-layer perceptrons}

Given a monotonically non-decreasing function $\varphi: {\mathbb{R}} \rightarrow [0,1]$, and an integer $n>1$, let 
function $\varphi_n: {\mathbb{R}} \rightarrow {\cal C}_n$ 
be defined as follows: 
\begin{align}\label{approx}
	\varphi_n(x) & = \left\{\begin{array}{ll}
						 0 & \mbox{ \ \ \ \  if } \varphi(x) \leq \frac{1}{2n} \\
						 \frac{i}{n} & \mbox{ \ \ \ \  if } \frac{2i -1}{2n} <  \varphi(x)  \leq \frac{2i +1}{2n}, \mbox{ for } 0<i<n \\
						1 &  \mbox{ \ \ \ \  if } \frac{2n -1}{2n} <  \varphi(x) 
					\end{array}\right.
\end{align} 
$\varphi_n(x)$ approximates  $\varphi(x)$ to the nearest value in ${\cal C}_n$.
The notions of {\em $\varphi_n$-coherence}, {\em $\varphi_n$-coherent model},  {\em canonical $\varphi_n$-coherent model}, {\em $\varphi_n$-coherent entailment} can be defined as 
in Definitions \ref{varphi-coherence}  and \ref{fm-entailment}, by replacing $\varphi$ with $\varphi_n$.
The above mentioned result concerning the existence of canonical models  also extends to  canonical $\varphi_n$-coherent models of weighted KBs (see Proposition 3 in the supplementary material for the paper, Appendix A).

In the following, 
we formulate the problem of $\varphi_n$-coherent entailment from a weighted $G_n \lc \tip$ ($\L_n \lc \tip$) knowledge base as a problem of computing preferred answer sets of an ASP program.
Verifying $\varphi_n$-coherent entailment of a typicality inclusion $\tip(C) \sqsubseteq D \; \theta \;\alpha$ from a weighted knowledge base $K$ (a subsumption problem),
would, in principle, require considering all typical $C$-elements in all possible canonical $\varphi_n$-coherent models of $K$, and checking whether they are all instances of $D$ with a degree $d$  such that $d \theta \alpha$.
We show that we can reformulate this problem as a problem of generating  answer sets representing $\varphi_n$-coherent models of the knowledge base,  and then selecting preferred answer sets, where a distinguished domain element $aux_C$ is intended to represent a typical $C$-element. 
For the selection of preferred answer sets, the ones maximizing the degree of membership of $aux_C$ in concept $C$, {\em asprin} \cite{BrewkaAAAI15} is used. 
Our proof method is sound and complete for the computation of $\varphi_n$-coherent  entailment.

Given  a weighted $G_n \lc \tip$ ($\L_n \lc \tip$) knowledge base
$K=\langle  {\cal T},$ $ {\cal T}_{C_1}, \ldots,$ $ {\cal T}_{C_k}, {\cal A} \rangle$,
we let  
$\Pi_{K,n}$ be the representation of $K$ in Datalog, where
$\mathit{val(v)}$ holds for $v$ a value in $\{0,1, \ldots, n\}$, which is intended to represent the value $\frac{v}{n}$ in ${\cal C}_n$;
$\mathit{nom(a)}$, $\mathit{cls(A)}$\footnote{  Uppercase is used here for concept names, to keep a DL-like notation, even though such names are ASP constants.}, 
are used for
$\mathit{a \in N_I}$, $\mathit{A \in N_C}$. We also have $\mathit{nom(auxc)}$\footnote{ Observe that the addition of further auxiliary constants to represent other domain elements in a model, as considered for $\el$ in the two-valued case \cite{iclp2021},
following the approach by Kr\"{o}tzsch in his Datalog materialization calculus \cite{KrotzschJelia2010},
is not needed here as neither existential nor universal restrictions are allowed.}.

Boolean concepts $C \sqcap D$, $C \sqcup D$, $\neg C$ are represented as $and(C',D')$, $or(C',D')$, $neg(C')$, where $C'$ and $D'$ are terms representing concepts $C$ and $D$;
$\mathit{subTyp(C',D',w')} $ represents a defeasible inclusion $( \mathit{T(C) \sqsubseteq D} , w)$, where $w'$ is an integer corresponding to $w \times 10^k$, for $w$ approximated to $k$ decimal places.
The concepts of interest, to be considered for
limiting grounding in the rules introduced later, 
are represented  (1) with assertions $concept(C')$, where $C'$ is the term for boolean concepts $C$ occurring 
in $K$ or in the formula to be verified (see later); (2) with rules implying that subconcepts are also of interest, e.g.:

$\;$ \ \  $ \mathit{
concept(A) \leftarrow concept(and(A,B)).
	} $

$\Pi_{K,n}$ also contains the set of  rules for generating $\varphi_n$-coherent models of $K$.
The valuation is encoded by a set of atoms of the form $\mathit{inst(x,A,v)}$, meaning that $\frac{v}{n} \in C_n$ is the degree of membership of $x$ in $A$.
The rule:

$\;$ \ \  $ \mathit{
1\{inst(X,A, V) : val(V)\}1 \ \leftarrow cls(A), nom(X).
	} $

\noindent	
generates alternative answer sets, corresponding to interpretations of each constant $\mathit{x}$ (either named individuals or $aux_C$), with different values $v$ corresponding to a membership degree $\frac{v}{n} \in {\cal C}_n$  in each atomic concept $A$.  

The valuation of complex boolean concepts $D$  is encoded by introducing a predicate $\mathit{eval(D, X, V)}$ to determine the membership degree $V$ of element $X$ in $D$. A rule is introduced for each boolean operator to encode its semantics. For $G_n \lc \tip$, the $\mathit{eval}$ predicate encodes the semantics of $\sqcap$, $\sqcup$ and $\neg$, based on G\"odel logic t-norm, s-norm and on involutive negation as follows:

$\;$ \ \ $ \mathit{
eval(A, X, V) \leftarrow cls(A), inst(X,A,V)}$.

$\;$ \ \   $ \mathit{
eval(and(A, B), X, V) \leftarrow concept(and(A, B)), eval(A,X,V1), eval(B,X,V1),}$  

$\;$ \ \ \ \ \ \ \ \ \ \ \ \    \ \ \ \ \ \ \ \ \ \ \ \   \ \ \ \ \ \ \ \ \ \ \ \  $ \mathit{min(V1, V2,V).
	} $
	
$\;$ \ \   $ \mathit{
eval(or(A, B), X, V) \leftarrow concept(or(A, B)), eval(A,X,V1), eval(B,X,V1),  }$  

$\;$ \ \ \ \ \ \ \ \ \ \ \ \    \ \ \ \ \ \ \ \ \ \ \ \   \ \ \ \ \ \ \ \ \ \ \ \  $ \mathit{ max(V1, V2,V).
	} $

$\;$ \ \ $ \mathit{
eval(neg( A), X, V) \leftarrow concept(neg (A)), eval(A,X,V1),   V= n-V1.
	} $

\noindent	
where the predicates $\mathit{min}$ and  $\mathit{max}$ are suitably defined.  
A similar evaluation function $\mathit{eval}$ can be defined for \L ukasiewicz combination functions.

To guarantee the satisfiability of $G_n \lc \tip$ axioms (assertions and inclusions) a set of constraints is added. For instance, for the assertion $C(a) \geq \alpha$ 
we add the constraint

$\;$ \ \ $\mathit{\bot \leftarrow eval(C',a,V), V < n\alpha .}$

\noindent
where $\mathit{C'}$ is the term representing concept $C$,
while for a strict $G_n \lc \tip$  inclusion $E \sqsubseteq D \geq \alpha$ we add the constraint

$\;$ \ \  $\mathit{\bot \leftarrow eval(E',X,V1), eval(D',X,V2), V1>V2, V2 < \alpha .}$

\noindent
and similarly for other axioms and for the $\L _n \lc \tip$ case. An answer set represents a $\varphi_n$-coherent interpretation if the following constraint is satisfied:

 $\mathit{\bot \leftarrow nom(X), dcls(Ci), eval(Ci,X,V), weight(X,Ci,W), }$
 
 $\;$ \ \ \ \ \ \ \   $\mathit{ valphi(n,W,V1),V != V1.}$

\noindent
where $\mathit{dcls(Ci)}$ is included in $\Pi_{K,n}$ for each distinguished class $C_i \in {\cal C}$; 
given that the weights $w_h^i$ are approximated to $k$ decimal places, 
argument $W$ for $\mathit{weight}$ corresponds to the integer $n \times W_i(x) \times  10^k$, 
and $valphi(n,W,V1)$ is defined (see below) to correspond to 
$ V1 =  n \times \varphi_n(W_i(x))) = n \times \varphi_n(W/(n \times 10^k)) $  again representing ${\cal C}_n$ with $\{0,1, \ldots, n\}$. 
Predicate $\mathit{weight}$  (for the weighted sum) could, in principle, be defined as follows:

$\;$ \ \  $ \mathit{
weight(X,C,W) \leftarrow dcls(C),nom(X), 
}$

$\;$ \ \ \ \ \ \ \ \ \ \ \ \    \ \ \ \ \ \  
$ \mathit{ 
W = \#sum\{ Wi*V,D : cls(D),eval(D,X,V), subTyp(C,D,Wi) \}.
} $

\noindent
but, for grounding reasons, it can be better defined with a rule for each individual distinguished class; 
such rules can be generated, for each distinguished concept $C_i$, from the set of weighted typicality inclusions $ {\cal T}_{C_i}$.
In particular, given 
${\cal T}_{C_i}=\{(\tip(C_i) \sqsubseteq D_{i,h},w^i_h), h=1,\ldots,k\}$, the following rule is introduced: 

$\;$ \ \  $ \mathit{
	weight(X,Ci',W) \leftarrow nom(X), W = Wi1 * Vi1 + \ldots + Wik * Vik,
}$

$\;$ \ \ \ \ \ \ \ \ \ \ \ \     
$ \mathit{  subTyp(Ci',Di1',Wi1), eval(Di1',X,Vi1), \ldots , }$

$\;$ \ \ \ \ \ \ \ \ \ \ \ \     
$ \mathit{   	subTyp(Ci',Dik',Wik), eval(Dik',X,Vik).   	}$

\noindent where $ \mathit{Ci'}$, $ \mathit{Di1'}$, $\ldots$, $ \mathit{Dik'}$
are the terms representing concepts $ \mathit{Ci}$, $ \mathit{Di1}$, $\ldots$, $ \mathit{Dik}$.

\noindent
Predicate $valphi$ can be defined with rules such as:

$\;$ \ \  $\mathit{valphi(n,W,0) \leftarrow num(W),W < k_1.}$

$\;$ \ \  $\mathit{valphi(n,W,1) \leftarrow num(W),W>= k_1, W < k_2.}$

$\;$ \ \ $\ldots$

$\;$ \ \  $\mathit{valphi(n,W,n) \leftarrow num(W),W > k_{n-1}.}$

\noindent	
where: 

$\;$ \ \  $ \mathit{
	num(W) \leftarrow nom(X),weight(X,C,W),dcls(C).
}$

\noindent
is used for limiting grounding of the previous rules, and $k_1, \ldots, k_{n-1}$ can be precomputed to be:

$k_1 = \lfloor w \rfloor$ where $w$ is such that $\varphi(w/(n \times 10^k)) = 1/2n$,

$k_2 = \lfloor w \rfloor$ where $w$ is such that $\varphi(w/(n \times 10^k)) = 3/2n$,

$\ldots$

$k_{n-1} = \lfloor w \rfloor$ where $w$ is such that $\varphi(w/(n \times 10^k)) = (2n-1)/2n$.

\noindent
The program $\Pi(K,n,C,D,\theta, \alpha)$ associated to the $G_n \lc \tip$ ($\L_n \lc \tip$)  knowledge base $K$ and a typicality subsumption 
$\tip(C) \sqsubseteq D \; \theta \; \alpha$ is composed of two parts, 
$\Pi(K,n,C,$ $D,\theta , \alpha)= \Pi_{K,n} \cup \Pi_{C,D, \theta \alpha}$.
We have already introduced the first one. $\Pi_{C,D,n,\theta,\alpha}$ contains the facts
 $ \mathit{nom(aux_C)}$ and $ \mathit{auxtc(aux_C,C')}$ and the rules:
 
$\;$ \ \  $\mathit{ok \leftarrow  eval(D', aux_C,V), V \theta \alpha n. }$ \ \ \ \ \  \ \ \ \ \  $\mathit{notok \leftarrow  not \;ok .}$

\noindent
where $\mathit{ok}$ is intended to represent that $aux_C$ satisfies the property that its membership degree $V$ in concept $D$ is such that $V \theta  \alpha$ holds. 

Given a query $\tip(C) \sqsubseteq D \; \theta \alpha$, we have to verify that, in all canonical $\varphi_n$-coherent models of the $G_n \lc \tip$ ($\L_n \lc \tip$)  knowledge base, all typical $C$-elements are $D$-elements with a certain degree $v$ (representing $v/n \in {\cal{C}}_n$)  such that $ v  \theta \alpha n$.
This verification is accomplished by generating answer sets corresponding to the $\varphi_n$-coherent models of the KB, and by selecting the preferred ones, in which the distinguished element $aux_C$ represents a typical $C$-element. 

Given two answer sets $S$ and $S'$ of $\Pi(K,n,C,D,\theta , \alpha)$,
{\em  $S$ is preferred to $S'$}  if the membership degree of $aux_C$ in concept $C$ is higher in $S$ than in $S'$, 
that is: if 
$\mathit{eval(C',aux_C,v1)}$ holds in $S$ and $\mathit{eval(C',aux_C,v2)}$ holds in $S'$, then $v1 > v2$.

This condition 
is encoded directly into a preference program for {\em asprin} as follows.
One such program requires defining when an answer set $S$ is preferred to $S'$ according to a preference $P$ (optimal solutions wrt such a preference can then be required with an $ \mathit{\#optimize}$ directive).
This is done by defining a predicate $better(P)$ for the case where $P$ is of the type being defined,
using predicates $holds$ and $holds'$ to check whether atoms hold in $S$ and $S'$, respectively. 
In this case the preference program, defining a ``concept wise'' preference, is simply as follows:  

\begin{tabbing}
	$ \mathit{\#program ~ preference(cwise).}$
	\\
	
	$ \mathit{
		better(P) \leftarrow
		preference(P,cwise),
		holds(auxtc(auxc,C)), 
	    betterwrt(C).} $ \\
	
	$ \mathit{
		betterwrt(C) \leftarrow
		holds(eval(C,auxc,V1)),holds'(eval(C,auxc,V2)),V1>V2.
		} $
	
\end{tabbing}

The query $\tip(C) \sqsubseteq D \; \theta \; \alpha$ is entailed from the knowledge base $K$ if, in all (maximally) preferred answer sets, $aux_C$ is an instance of concept $D$ with a membership degree $v$ (representing $v/n \in {\cal{C}}_n$) such that $v \theta \alpha n$ holds; i.e., if $\mathit{ok}$ holds in all preferred answer sets, or,
equivalently, $\mathit{notok}$ does not hold in any of them.
In fact, we can prove that this corresponds to verifying that  $D$ is satisfied in all $<_C$-minimal $C$-elements in all canonical $\varphi_n$-coherent models of the knowledge base:

\begin{proposition}  \label{AS to pref-models}
Given a $G_n \lc \tip$ ($\L_n \lc \tip$)  knowledge base $K$, 
the query $\tip(C) \sqsubseteq D \; \theta \; \alpha$ is falsified in 
some canonical $\varphi_n$-coherent model of $K$ if and only if there is a preferred answer set $S$ of the program $\Pi(K,C,D,n,\theta,\alpha)$ containing $eval(D',aux_C,v)$ such that $v \theta \alpha n$ does not hold (and then containing $\mathit{notok}$).
\end{proposition}
The proof can be found in the supplementary material for the paper, Appendix B. 
It exploits the existence of $\varphi_n$-coherent canonical models, for KBs having a $\varphi_n$-coherent model
(Proposition \ref{prop:existence_canonical_model} in Appendix A). 
Appendix B also contains a proof 
of the following upper bound on the complexity of $\varphi_n$-coherent entailment.

\begin{proposition}   \label{prop:upper bound}
$\varphi_n$-coherent entailment from a weighted  $G_n \lc \tip$ ($\L_n \lc \tip$) knowledge base is in  $\Pi^p_2$.
\end{proposition} 

As a proof of concept, the approach has been experimented for the weighted $G_n \lc \tip$ KBs corresponding  to two of the trained multilayer feedforward network for the MONK's problems 
\cite{monk}, namely, the network for
problem 1 and the second network for problem 3 . 
The networks have 17 non-independent binary inputs, corresponding to values of 6 inputs having 2 to 4 possible values; such inputs are features of a robot, e.g., head shape and body shape being round, square or octagon, and jacket color being red, yellow, green or blue. The network for problem 1 (Figure \ref{networkmonk1}) has 3 hidden units ($h1,h2,h3$) and an output unit ($o$);
the one for problem 3 has 2 hidden units.

\begin{figure}[t]
	\centering
	\includegraphics[width=0.8\textwidth]{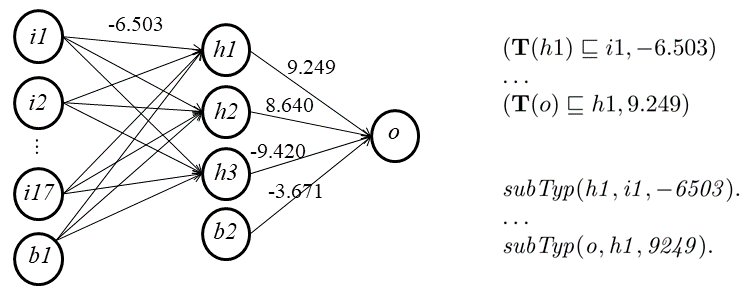}
	\caption{ The network for MONK's problem 1, with some of the weights after training (using 3 decimal digits), two of the corresponding typicality inclusions and their ASP representation.\label{networkmonk1} }
	
\end{figure}

In the two problems, the trained networks learned to classify inputs satisfying two formulae, respectively, $F1$ and $F3$, which are boolean combinations of the inputs.
In particular, $F1$ is $\mathit{jacket\_color\_red \ or \ head\_shape = body\_shape}$
and, in terms of the classes $\mathit{i1, \ldots , i17}$ corresponding to the binary inputs, it is:

$\;$ \ \ \ \ $\mathit{F1 \equiv i12  \sqcup  (i1 \sqcap i4) \sqcup  (i2 \sqcap i5) \sqcup  (i3 \sqcap i6) }$

\noindent
($\mathit{i12}$ is $\mathit{jacket\_color\_red}$, $\mathit{i1}$ is $\mathit{head\_shape\_round}$, $\mathit{i4}$ is $\mathit{body\_shape\_round}$, etc.).

The approach described above has been applied, using values $0$ and $1$ as possible values for classes associated with input nodes,
rather than all values in ${\cal C}_n$.
The networks are feedforward, then for a choice of values for input nodes,
there is only one choice of values in ${\cal C}_n$ for non-input nodes satisfying the constraint for 
$\varphi_n$-coherent interpretations (then the number of answer sets of $\Pi(K,C,D,n,\theta,\alpha)$ is given by the possible combinations of input values and does not depend on $n$).

For the trained network for problem 1, 
for, e.g., $n=5$, the formula
$\tip(o) \sqsubseteq F1 \geq 1$ can be verified; 
$o$ is the concept name associated with the output unit.
That is, the  $G_5 \lc \tip$ knowledge base entails that the typical $o$-elements satisfy $F1$.
The formula can also be verified for $n=1,3,9$  with minor variations on the running time (all below 10 s).
This result is explainable (also for n=1), as an input was classified by the network as class member if the output was $\geq 0.5$
and, for problem 1, the network learned the concept with 100\% accuracy.

Stronger variants of $F1$ have also been considered, to check that the network learned $F1$ but not such variants.
For the following variants with one less disjunct:

$\;$ \ \ \ \ $\mathit{F1' \equiv  i12  \sqcup  (i1 \sqcap i4) \sqcup  (i2 \sqcap i5) } \ \ \ \ \ \ \ \ $
$\mathit{F1'' \equiv (i1 \sqcap i4) \sqcup  (i2 \sqcap i5) \sqcup  (i3 \sqcap i6) }$

\noindent
the formulae $\tip(o) \sqsubseteq F1' \geq 1$ and $\tip(o) \sqsubseteq F1'' \geq 1$ are indeed not entailed for $n=1,3,5,9$.

An important issue in analysing a trained network is also associating a meaning to hidden nodes.
The following formulae have been verified for $n=1,3,5,9$ for hidden nodes  $\mathit{h1,h2,h3}$:

$\;$ \ \ \ \ $\tip(h1) \sqsubseteq i12  \sqcup  (\neg i1 \sqcap \neg i4) \geq 1$ 

$\;$ \ \ \ \  $\tip(h2) \sqsubseteq i12  \sqcup  (\neg i3 \sqcap \neg i6) \geq 1$ 

$\;$ \ \ \ \ $\tip(h3) \sqsubseteq \neg i12  \sqcup  (i2 \sqcup i5) \geq 1$

In problem 3, there was noise (some misclassifications) in the training set. Then the accuracy of the trained network is not 100\%.
However, the trained network produces no false positives.
Therefore, the formula $\tip(o) \sqsubseteq F3 \geq 1$ can be verified for $n=1,3,5,9$, 
where  $F3$ is $\mathit{(jacket\_color\_red \ and \ holding\_sword) \ or \ (not \ jacket\_color\_blue \ and}$  
 $\mathit{ \ not \ body\_shape\_octagon)}$.    
Since there are false negatives,  the formula $\tip(\neg o) \sqsubseteq \neg F3 \geq 1$ is not entailed for $n=1$ but, for instance, it is for $n=5$.

\normalcolor

\section{Conclusions}

The ``concept-wise'' multipreference semantics (both in the two-valued and in the fuzzy case) has recently been proposed as a logical semantics of  MultiLayer Perceptrons (MLPs) \cite{JELIA2021}. In this paper we consider weighted conditional $\alc$ knowledge bases  in the finitely many-valued case, under a coherent, a faithful and a $\varphi$-coherent semantics, the last one being suitable to characterize the stationary states of MLPs. 
For the boolean fragment $\lc$ of $\alc$ we exploit  ASP and {\em asprin}  \cite{BrewkaAAAI15} for reasoning under $\varphi$-coherent entailment, by  restricting to canonical models of the knowledge base. 
We have proven soundness and completeness of ASP encoding for the finitely many-valued case and provided an upper complexity bound.
As a proof of concept, we have experimented the proposed approach  for checking properties of some trained neural networks for the MONK's problems \cite{monk}.

Undecidability results for fuzzy DLs with general inclusion axioms 
\cite{CeramiStraccia2011,BorgwardtPenaloza12} 
motivate the investigation of many-valued approximations of fuzzy multipreference entailment. 
 The choice of $\lc$ is motivated by the fact it is sufficient to encode a  neural network as a weighted KB as well as to formulate boolean properties of the network. 
This work is a first step towards the definition of proof methods for reasoning from weighted KBs under a finitely many-valued 
preferential semantics 
 in more expressive or lightweight DLs. 
 For $\el$, the two-valued case has been studied in previous work \cite{iclp2021}.
 
The encoding of a neural network as a conditional knowledge base opens the possibility of combining  empirical knowledge with elicited knowledge, e.g.,
in the form of strict inclusions and definitions.
Much work has been devoted, in recent years, to the combination 
of neural networks and symbolic reasoning (see the survey by Lamb et al. \cite{GarcezGori2020}), leading to the definition of new computational models
and to extensions of logic programming languages
with neural predicates. 
The relationships between normal logic programs and connectionist network have been investigated by Garcez and Gabbay \cite{CLIP99} 
and by Hitzler et al. \cite{HitzlerJAL04}.
A correspondence between neural networks and gradual argumentation semantics has been recently investigated by Potyka  \cite{PotykaAAAI21}
by studying the semantic properties and the convergence conditions of a MLP-based bipolar semantics.  %
The correspondence between neural network models and fuzzy systems has been first investigated by Kosko in his seminal work  \cite{Kosko92}.
A fuzzy extension of preferential logics has been studied by Casini and Straccia \cite{CasiniStraccia13_fuzzyRC} based on Rational Closure.

While using preferential logic for the verification of properties of neural networks is a general (model agnostic) approach, first proposed for 
SOMs \cite{CILC2020,JLC2022},
whether it is possible to extend the logical encoding of MLPs as weighted conditional KBs  to other network models 
is a subject for future investigation. The development of a temporal extension of weighted conditional KBs to capture the transient behavior of MLPs is also an interesting direction to extend this work.

\medskip

{\bf Acknowledgement:} 
This research is partially supported by INDAM-GNCS Project 2020.


\newpage

\begin{center}
{\Large \bf Appendix} 
\end{center}

\begin{appendix}

\section{Existence of canonical $\varphi$-coherent/$\varphi_n$-coherent models}

For canonical $\varphi$-coherent and $\varphi_n$-coherent models, we can prove the following result.

\begin{proposition} 
\label{prop:existence_canonical_model}
A  weighted $G_n \lc \tip$ ($\L _n \lc \tip$) knowledge base $K$ has a canonical $\varphi$-coherent ($\varphi_n$-coherent) model, if it has a $\varphi$-coherent ($\varphi_n$-coherent) model.
\end{proposition}

\begin{proof}  [{sketch}]
We prove the result for $\varphi$-coherent models of a weighted $G_n \lc \tip$ ($\L _n \lc \tip$) knowledge base $K$. The proof for $\varphi_n$-coherent models is the same.

Given a  weighted $G_n \lc \tip$ ($\L _n \lc \tip$) knowledge base $K=\langle  {\cal T},$ $ {\cal T}_{C_1}, \ldots,$ $ {\cal T}_{C_k}, {\cal A}  \rangle$, let  $I_0 = \la \Delta_0, \cdot^{I_0}\ra$ be a $\varphi$-coherent model of $K$.
A canonical $\varphi$-coherent  model for $K$ can be constructed starting from the model $I_0$ as follows.

First, let $\lingconc$ be the set of all concept names $B \in  N_C$ occurring in $K$.  The set $\lingconc$ is finite.
Considering the finitely many concept names $B$  in $\lingconc$ and the finitely many truth degrees in ${\cal C}_n=\{0,\frac{1}{n}, \ldots, \frac{n-1}{n},1\}$,
there are finitely many valuations $e$ assigning a membership degree in ${\cal C}_n$ to each concept  name $B$ in $\lingconc$, i.e., such that $e(B) \in {\cal C}_n$, for each $B \in \lingconc$.

Let us call  $e_1, \ldots, e_k$ all such possible valuations over $\lingconc$.
Starting from the $\varphi$-coherent model $I_0$ of $K$, we extend the domain $\Delta_0$ by possibly introducing new domain elements $x_i$, one for each valuation $e_i$, provided valuation $e_i$ is present in some $\varphi$-coherent model of $K$, but not in $I_0$.

We say that valuation {\em $e_i$ is present in an interpretation $I= \la \Delta, \cdot^{I}\ra$} of $K$ if there is a domain element $x \in \Delta$ such that 
$B^I(x)=e_i(B)$, for all concept names $B \in \lingconc$.

We say that a valuation {\em $e_i$ is missing in $I_0$ for $K$} if it is present in some $\varphi$-coherent model $I$ of $K$, but it is not present in $I_0$.

Let us define a new interpretation $I^* = \la \Delta^*, \cdot^{I^*}\ra$ with  domain
$$\Delta^*= \Delta_0 \cup  \{x_i \mid \mbox{ valuation $e_i$  is missing in $I_0$ } \}$$
$\Delta^*$ contains a new element $x_i$ for each valuation $e_i$ which is missing in $I_0$.

The interpretation of individual names in $I^*$ remains the same as in $I_0$.
The interpretation of concepts in $I^*$ is defined as follows:

- $B^{I^*}(x)= B^{I_0}(x)$ for all $x \in \Delta_0$, for all $B \in N_I$;

- $B^{I^*}(x_i)= e_i(B)$, for all $B \in \lingconc$;

- $B^{I^*}(x_i)= B^{I_0}(z)$, for all $B \in N_C$ s.t. $B \not \in \lingconc$,

\noindent
where $z$ is an arbitrarily chosen domain element in $\Delta_0$.
Informally, the interpretation of concepts in $I^*$ is defined as in $I_0$ on the elements of $\Delta_0$, while  it is given by valuation $e_i$ for the added domain element $x_i$, for the concept names $B$ in $\lingconc$. For  the concept names $B$ not occurring in $K$ the interpretation of $B$ in $x_i$ is taken to be the same as the interpretation in $I_0$ of $B$ in some domain element $z \in \Delta_0$.

We have to prove that $I^*$ satisfies all  
$G_n \lc \tip$ ($\L _n \lc \tip$) inclusions and assertions in $K$ and that it is a $\varphi$-coherent model of $K$. 
$I^*$ also satisfies condition (ii) in Definition 6 by construction, as all the finitely many possible valuations $e_i$ over $\lingconc$, which are present in some $\varphi$-coherent model of $K$, are considered.

To prove that $I^*$ satisfies all assertions in ${\cal A}$, let $C(a) \;\theta \alpha$ be in $K$.
Then all the concept names in $C$ are in $\lingconc$. By construction  $a^{I^*}=a^{I_0}$.
Furthermore,
it can be proven that $(E(a))^{I^*}= E^{I^*}(a^{I_0}) = E^{I_0}(a^{I_0}) = (E(a))^{I_0}$ holds for all concepts $E$ occurring in $K$
 (the proof is by induction on the structure of concept $E$). 
 Hence, $(C(a))^{I^*}= (C(a))^{I_0}$.
As $C(a) \; \theta \alpha$ is satisfied in $I_0$, then $(C(a))^{I_0} \;\theta \alpha$ holds, and $(C(a))^{I^*} \; \theta \alpha$ also holds.

To prove that $I^*$ satisfies all concept inclusions in ${\cal T}$, let $C \sqsubseteq D \;\theta \alpha$ be in $K$.
Then all the concept names in $C$ and in $D$ are in $\lingconc$.
We have to prove that, for all $x \in \Delta^*$, $C^{I^*}(x) \rhd D^{I^*}(x)  \;\theta \alpha$. We prove it by cases.

For the case $x \in \Delta_0$.
It can be proven that, for all $x \in \Delta_0$
$ E^{I^*}(x) = E^{I_0}(x)$ holds for all concepts $E$ occurring in $K$
 (the proof is by induction on the structure of concept $E$). 
 Therefore, $ C^{I^*}(x) = C^{I_0}(x)$ and  $ D^{I^*}(x) = D^{I_0}(x)$ hold.
As axiom $C \sqsubseteq D \;\theta \alpha$ is satisfied in $I_0$, 
$C^{I_0}(x) \rhd D^{I_0}(x)  \;\theta \alpha$  holds. Therefore,
$C^{I^*}(x) \rhd D^{I^*}(x)  \;\theta \alpha$ also holds.

For  $x \not \in \Delta_0$, $x=x_i$ for some $i$. By construction, as $e_i$ is missing in $I_0$, $e_i$ must be present in some interpretation $I'= \la \Delta', \cdot^{I'}\ra$ of $K$, i.e., there is a domain element $y \in \Delta'$ such that 
$B^{I'}(y)=e_i(B)$, for all concept names $B \in \lingconc$.  
It can be proven that, 
$ E^{I^*}(x_i) = E^{I'}(y)$ holds for all concepts $E$ occurring in $K$.
(the proof is by induction on the structure of concept $C$). 
 Therefore, $ C^{I^*}(x_i) = C^{I'}(y)$ and  $ D^{I^*}(x_i) = D^{I'}(y)$ hold.
As axiom $C \sqsubseteq D \;\theta \alpha$ is satisfied in $I'$, 
$C^{I'}(y) \rhd D^{I'}(y)  \;\theta \alpha$  holds. Therefore,
$C^{I^*}(x_i) \rhd D^{I^*}(x_i)  \;\theta \alpha$ also holds.

The proof that $I^*$ is a $\varphi$-coherent model of $K$ is similar. \hfill $\Box$
\end{proof}


\section{Proof of Proposition 2} \label{appendix:Prop2}

\begin{lemma} \label{AS to models}
Given a  weighted $G_n \lc \tip$  ($\L _n \lc \tip$)  knowledge base $K=\langle  {\cal T},$ $ {\cal T}_{C_1}, \ldots,$ $ {\cal T}_{C_k}, {\cal A}  \rangle$ over the set of distinguished concepts ${\cal C} =\{C_1, \ldots, C_k\}$,  
and a subsumption  $C \sqsubseteq D \theta \alpha$, we can prove the following: 
\begin{itemize}
\item[(1)]
if there is an answer set $S$ of the ASP program $\Pi(K, n, C,D, \theta, \alpha)$, 
such that $eval(E',$ $aux_C, v) \in S$, for some concept $E$ occurring in $K$,
then there is a $\varphi_n$-coherent model $I=\langle \Delta, \cdot^I \rangle $ for $K$  
and an element $x \in \Delta$ such that $E^I(x)=\frac{v}{n}$.

\item[(2)]
if there is a $\varphi_n$-coherent model $I=\langle \Delta, \cdot^I \rangle $ for $K$  
and an element $x \in \Delta$ such that $E^I(x)=\frac{v}{n}$,  for some concept $E$ occurring in $K$ and some $v \in \{0,\ldots,n\}$,
then there is an answer set $S$  of the ASP program $\Pi(K, n, C,D, \theta, \alpha)$, 
such that $eval(E',$ $aux_C, v) \in S$.
\end{itemize}
\end{lemma}

\begin{proof} [{sketch}]
We prove the lemma for $G_n \lc \tip$  (the proof for $\L _n \lc \tip$ is similar). 

For part (1), given an answer set $S$ of the program $\Pi(K, n, C,D, \theta, \alpha)$
such that $eval(E',$ $aux_C,$ $v) \in S$, for some concept $E$ occurring in $K$,
we construct a $\varphi_n$-coherent model $I=\langle \Delta, \cdot^I \rangle $ of $K$  such that $E^I(x)=\frac{v}{n}$.
Let $N_I$ and $N_C$ be the set of named individuals and named concepts in the language. We take as the domain $\Delta$ of $I$ the set of constants including all the named individuals $d \in N_I$ occurring in $K$ 
plus an auxiliary element $z_C$, 
 i.e., $\Delta = \{ e \mid e \in N_I \mbox{ and $e$ occurs in $K$}\} \cup \{z_C\}$.

For each element $e \in \Delta$, we define a projection $\prj(e)$ to a corresponding ASP constant as follows:

- $\prj(z_{C})=aux_{C}$;

- $\prj(e)=e$, if $e \in N_I$ and $e$ occurs in $K$.

\noindent
Note that, for all $e \in \Delta$, $\mathit{nom(\prj(e))\in S}$ by construction of the program $\Pi(K, n, C,D,$ $ \theta, \alpha)$.

\noindent
The interpretation of individual names in $e \in N_I$ over $\Delta$ is defined as follows:

- \ $e^I = e$,  if $e$ occurs in $K$;

- \ $e^I = a$,  otherwise,

\noindent
where $a$ is an arbitrarily chosen element in  $\Delta$. 

The interpretation of named concepts $A \in N_C$ is as follows:
\begin{quote}
 -  $A^{I}(e)=\frac{v}{n}$ iff $\mathit {inst}(\prj(e), A, v) \in S$, for all $e \in \Delta$, if  $A$ occurs in $K$;
 
  - $A^{I}=B^I$, if $A$ does  not occur in $K$,
 
\end{quote}
where $B$ is an arbitrarily chosen concept name occurring in $K$. 

This defines a $G_n \lc \tip$ interpretation.
Let us prove that $I=\langle \Delta, \cdot^I \rangle $ is a $\varphi_n$-coherent model of $K$.

Assume that the $w_h^i$ are approximated to $k$ decimal places. From the definition of the $\mathit{eval}$ predicate, one can easily prove that the following statements hold, for all concepts $C$ and distinguished concepts $C_i$  occurring in $K$, and for all $e \in \Delta$:
\begin{itemize}
\item
$\mathit{C^I(e)=\frac{v}{n}}$ if and only if $\mathit{eval(C',\prj(e),v) \in S}$;

\item
$\mathit{weight(\prj(e), C'_i, w)} \in S$ if and only if $w= W_i(e) \times 10^k \times n$;

\item
$\mathit{valphi(n,w,v) \in S}$ if and only if $\frac{v}{n}=\varphi_n(w/(10^k \times n))$;

\item
$\varphi_n(\sum_{h} w_h^i  \; D_{i,h}^I(e))=\frac{v}{n}$   if and only if   $\mathit{weight(\prj(e), C'_i, w)} \in S$ and

$\;$ \ \ \ \ \ \ \   \ \ \ \ \ \ \   \ \ \ \ \ \ \   \ \ \ \ \ \ \   \ \ \ \ \ \ \   \ \ \ \ \ \ \   \ \ \ \ \ \ \   \ \ \ \ \ \ \    \ \ \ \ \ \  $\mathit{valphi(n,w,v) \in S}$;

\end{itemize}
where $C'$ is the ASP encoding of concept $C$, and $C'_i$ is the ASP encoding of concept $C_i$.

First we have to prove that $I$ satisfies the $G_n \lc \tip$  inclusions in TBox ${\cal T}$ and  assertions in ABox ${\cal A}$. 
Suitable constraints in $\Pi(K, n, C,D, \theta, \alpha)$ guarantee that all assertions are satisfied.
For instance, for assertion $C(a) \geq \alpha$, the constraint

$\;$ \ \  $\mathit{\bot \leftarrow eval(C',a,V), V < \alpha n}$,

\noindent
is included in the ASP program and
we know that it is not the case that $\mathit{eval(C',a,v)} \in S$ and $v < \alpha n$ holds.
By the equivalences above, it is not the case that $\mathit{C^I(a^I)=\frac{v}{n}}$ and $\frac{v}{n} < \alpha$ holds. Hence, $\mathit{C^I(a^I) < \alpha}$ does not hold.

For a $G_n \lc \tip$ concept inclusion of the form $E \sqsubseteq D \geq \alpha$, the following constraint 

$\;$ \ \  $\mathit{\bot \leftarrow eval(E',X,V1), eval(D',X,V2), V1>V2, V2 < \alpha n}$. \\
holds for $X$ instantiated with any constant $a$ such that $\mathit{nom(a)\in S}$. Hence,  it is not the case that, for any such an $a$,  $\mathit{eval(E',a,v_1), eval(D',a,v_2)}$ belong to $S$ and that $v_1>v_2$ and  $v_2 < \alpha n$ hold.
Therefore,  it is not the case that for some $d \in \Delta$ $\mathit{E^I(d)=\frac{v_1}{n}}$, $\mathit{D^I(d)=\frac{v_2}{n}}$ and that $\frac{v_1}{n} > \frac{v_2}{n}$,
$\frac{v_2}{n} < \alpha$ hold. That is, $E \sqsubseteq D \geq \alpha$ is satisfied in $I$. Similarly, for other concept inclusions in ${\cal T}$.

The interpretation $I$ represents a $\varphi_n$-coherent model of $K$ if
$$C^I_i(e)= \varphi_n(\sum_{h} w_h^i  \; D_{i,h}^I(x))$$ 
holds for all $e \in \Delta$ and for all distinguished concepts $C_i$.
We prove that this condition holds for $I$. In fact, all ground instances of the following constraint

$\;$ \ \  $\mathit{\bot \leftarrow nom(x), dcls(Ci), eval(Ci,x,V), weight(x,Ci,W),} $

$\;$ \ \ \ \ \ \ \   \ \ \ \   $\mathit{valphi(n,W,V1), V != V1.}$

\noindent
must be statisfied in $S$.
Hence, there cannot be a distinguished concept $C_i$ and a  constant $a$  with $nom(a) \in S$, such that 
$\mathit{eval(C'_i,a,v), weight(a,C'_i,w)}$ and $\mathit{valphi(n,w,v_1)}$ belong to $S$, and  $v_1\neq v$.
Thus, it is not the case that, for some $e \in \Delta$, $\mathit{C_i^I(e)=\frac{v}{n}}$, $\varphi_n(\sum_{h} w_h^i  \; D_{i,h}^I(e))=\frac{v_1}{n}$, and
$v \neq v_1$.

By construction of the  $\varphi_n$-coherent model $I=\langle \Delta, \cdot^I \rangle $ of $K$, 
if $eval(E',aux_C,$ $v) \in S$, for some concept $E$ occurring in $K$, as $aux_C= \prj(z_C)$,
it follows that $E^I(z_C)=\frac{v}{n}$ holds in $I$.


For part (2), assume that there is 
a $\varphi_n$-coherent model $I=\langle \Delta, \cdot^I \rangle $ for $K$  
and an element $x \in \Delta$ such that $E^I(x)=\frac{v}{n}$,  for some concept $E$ occurring in $K$.
We can construct an answer set $S$ of the ASP program $\Pi(K, n, C,D, \theta, \alpha)$, 
such that  $eval(E',aux_C,v) \in S$.

Let us define a set of atoms $S_0$ by letting: 
 \begin{quote}
$\mathit{inst(a,A,v)} \in S_0$ if  $ A^I(a^I)=\frac{v}{n}$  in the model $I$, and 

$\mathit{inst(aux_C,A,v)} \in S_0$ if  $ A^I(x)=\frac{v}{n}$  in the model $I$,
 \end{quote}
for all concept names $A$ occurring in $K$,
and for all $a \in N_I$ such that $\mathit{nom(a)}$ is in $\Pi(K, n, C,D, \theta, \alpha)$. 
Nothing else is in $S_0$.

Let $\Pi_1$ be the set of ground instances of all definite clauses and facts in $\Pi(K, n, C,$ $D, \theta, \alpha)$, i.e., the grounding of all rules in  $\Pi(K, n, C,D, \theta, \alpha)$ with the exception of rule (r1), of the constraints and of the rule for $\mathit{notok}$.

Let $S$ be the set of all ground facts which are derivable from program $\Pi_1 \cup S_0$ plus, in addition, $\mathit{notok}$ in case $\mathit{ok}$ is not derivable.
It can be proven that, for all constants $a \in N_I$ such that $\mathit{nom(a)}$ is in $\Pi(K, n, C,D, \theta, \alpha)$ and  for all concepts $E$ occurring in $K$ (including subconcepts):
 \begin{quote}
 $\mathit{eval(E', a ,v)} \in S$ if and only if  $E^I(a^I)=\frac{v}{n}$

 $\mathit{eval(E', aux_C ,v)} \in S$ if and only if  $E^I(x)=\frac{v}{n}$
\end{quote} 
where $E'$ is the ASP encoding of concept $E$. 
Furthermore, for all distinguished concepts $C_i$.:
\begin{quote}

$\varphi_n(\sum_{h} w_h^i  \; D_{i,h}^I(a^I))=\frac{v}{n}$   if and only if   $\mathit{weight(a, C'_i, w)} $ and $\mathit{valphi(n,w,v)}$ are in $S$;

\end{quote}
where $C'_i$ is the ASP encoding of concept $C_i$.

$S$ is a consistent set of ground atoms, i.e., $\bot \not \in S$. 
 Notice that our ASP encoding does not make use of explicit negation and $S$ cannot contain complementary literals.
It can be proven that all constraints in $\Pi(K, n, C,D, \theta, \alpha)$ are satisfied by $S$.
Consider, for instance the constraint $\mathit{\bot \leftarrow eval(C',a,V), V < \alpha n}$, associated to an assertion $C(a) \theta \alpha$  in $K$.
As the assertion  $C(a) \theta \alpha$ is in $K$, it must be satisfied in the model $I$
and, for some $v$, $C^I(a^I)=\frac{v}{n}$ and $\frac{v}{n} \theta \alpha$.
Hence,  atom $\mathit{eval(C',a,v)}$ is in $S$ and $v \theta \alpha n$ holds, so that the  constraint associated to assertion $C(a) \theta \alpha$
 in $\Pi(K, n, C,D, \theta, \alpha)$ is satisfied in $S$.

Similarly, we can prove that all other constraints, those associated to the inclusion axioms and those that encode the $\varphi_n$-coherence condition
are as well satisfied in $S$, as the interpretation $I$ from which we have built the set $S$ is a $\varphi_n$-coherent model of $K$, and satisfies all inclusion axioms in ${\cal T}$.

We can further prove that all ground instances of the rules in $\Pi(K,C,D,n, \theta, \alpha)$ are satisfied in $S$.
This is obviously true for all the definite clauses which have been used deductively to determine  $S$ starting from $S_0$ by forward chaining.
This is also true for the choice rule (r1),

$\;$ \ \ \ \ \ \ \  $ \mathit{
1\{inst(x,A, V) : val(V)\}1 \ \leftarrow cls(A), nom(x).
	} $

\noindent
as the choice of atoms $\mathit{inst(a,A,v)}$ we have included in $S_0$ is one of the possible choices allowed by rule (r1).
We have already seen that all constraints in  $\Pi(K, n, C,D,$ $ \theta, \alpha)$ are satisfied in $S$.
Finally, also rule $\mathit{notok \leftarrow  not \;ok .}$ is satisfied in $S$, as we have added $\mathit{notok}$ in $S$ in case $\mathit{ok} \not \in S$.

We have proven that $S$ is a consistent set of ground atoms and all ground instances of the rules in $\Pi(K, n, C,D, \theta, \alpha)$ are satisfied in $S$.
To see that $S$ is an answer set of $\Pi(K, n, C,D, \theta, \alpha)$, it has to be proven that all literals in $S$ are supported in $S$.
Informally, observe that, all literals (facts) in $S$ can be obtained as follows: first by applying the choice rule (r1), which supports the choice of the atoms $\mathit{inst(a,A,v)}$ in $S_0$ (and in $S$); then by  exhaustively applying all ground definite clauses in $\Pi(K, n, C,D, \theta, \alpha)$ (by forward chaining) and, finally, by applying rule $\mathit{notok \leftarrow  not \;ok .}$, to conclude $\mathit{nottok}$ if $\mathit{ok} \not \in S$. 
 
 From the hypothesis, for element $x \in \Delta$ it holds that $E^I(x)=\frac{v_1}{n}$. Then, we can conclude that $eval(E',aux_C,$ $v) \in S$, which concludes the proof. \hfill $\Box$
\end{proof}

\medskip


\noindent
{\em Proposition \ref{AS to pref-models}}\\
Given a weighted $G_n \lc \tip$ ($\L _n \lc \tip$) knowledge base $K$, 
a query $\tip(C) \sqsubseteq D \theta \alpha$ is falsified in 
some canonical $\varphi_n$-coherent model of $K$ if and only if there is a preferred answer set $S$ of the program $\Pi(K,C,D,n, \theta, \alpha)$ such that $eval(D',aux_C,v)$ is in $S$ and $v \theta \alpha n$ does not hold.

\begin{proof}[{sketch}]
Let $K=\langle  {\cal T},$ $ {\cal T}_{C_1}, \ldots,$ $ {\cal T}_{C_k}, {\cal A}  \rangle$ be a $G_n \lc \tip$  knowledge base over the set of distinguished concepts ${\cal C} =\{C_1, \ldots, C_k\}$.
We prove the two directions:
\begin{itemize}
\item[(1)]
if there is a canonical  $\varphi_n$-coherent model $I=(\Delta, \cdot^I)$ of $K$  
that falsifies  $\tip(C) \sqsubseteq D  \theta \alpha$, then there is a preferred answer set $S$ 
of $\Pi(K, n, C,D, \theta, \alpha)$ 
such that, for some $v$, $eval(D', aux_C, v)  \in S$ and  $v  \theta \alpha n$ does not hold.

\item[(2)]
if there is a preferred answer set $S$ of $\Pi(K, n, C,D, \theta, \alpha)$   
such that, for some $v$, $eval(D',$ $ aux_C, v)$ is in  $S$ and  $v  \theta \alpha n$ does not hold,
then there is a canonical  $\varphi_n$-coherent model $I=(\Delta, \cdot^I)$ of $K$  
that falsifies  $\tip(C) \sqsubseteq D \theta \alpha$. 
\end{itemize}
We prove (1) and (2) for $G_n \lc \tip$  (the proof for $\L _n \lc \tip$ is similar). 

For part (1), assume that there is a canonical  $\varphi_n$-coherent model $I=(\Delta, \cdot^I)$ of $K$  
that falsifies  $\tip(C) \sqsubseteq D  \theta \alpha$.
Then, there is  some $x \in \Delta$, such that $x \in min_{<_C}(C_{>0}^I)$, $D^I(x)=\frac{v}{n}$ and it does not hold that $\frac{v}{n} \theta \alpha$.

By Lemma \ref{AS to models}, part (2), we know that there is an answer set $S$ 
of the ASP program $\Pi(K, n ,C,D, \theta, \alpha)$  
such that $eval(D',aux_C, v) \in S$. Clearly, $v \theta \alpha n$ does not hold.
We have to prove that $S$ is a preferred answer set of $\Pi(K, n, C,D, \theta, \alpha)$.

By construction of $S$, for all constants $a \in N_I$ such that $\mathit{nom(a)}$ is in $\Pi(K, n, C,D,$ $ \theta, \alpha)$, we have:

$\;$ \ \ \ $\mathit{inst(a,A,v)} \in S$ if  $ A^I(a^I)=\frac{v}{n}$ in model $I$,  

$\;$ \ \ \ $\mathit{inst(aux_C,A,v)} \in S$ if  $ A^I(x)=\frac{v}{n}$ in model $I$,  \\
for all concept names $A$ occurring in $K$.

Suppose by absurd that $S$ is not preferred among the answer sets of $\Pi(K, n, C,D,$ $ \theta, \alpha)$.
Then there is another answer set $S'$ which is preferred to $S$. This means that if $\mathit{eval(C',}$ $\mathit{aux_C,v_1)} \in S$ and $\mathit{eval(C',aux_C,v_2)} \in S'$, then $v_2>v_1$.

By construction of $S$ (see Lemma \ref{AS to models}, part (2)), from $\mathit{eval(C', aux_C,v_1)} \in S$ it follows that  $C^I(x)= \frac{v_1}{n}$ in the $\varphi_n$-coherent model $I$ of $K$.

As $S'$ is also an answer set of $\Pi(K, n, C,D, \theta, \alpha)$, 
by Lemma \ref{AS to models}, part (1), from $S'$ we can build a $\varphi_n$-coherent model $I'=\langle \Delta', \cdot^{I'} \rangle $ of $K$ such that $C^{I'}(z_C)= \frac{v_2}{n}$, for $z_C \in \Delta'$.

As $I$ is a canonical model, there must be an element $y \in \Delta$ such that $B^I(y)=B^{I'}(z_C)$, for all concepts $B$.
Therefore,  $C^I(y)=C^{I'}(z_C) =\frac{v_2}{n}$. 
As $ \frac{v_2}{n}  > \frac{v_1}{n}$, this contradicts the hypothesis that 
$x \in min_{<_C}(C_{>0}^I)$.
Then, $S$ must be preferred among the answer sets of $\Pi(K,n,C,D,  \theta, \alpha)$.

\medskip

For part (2), let us assume that there is a preferred answer set $S$ of $\Pi(K, n, C,D, \theta, \alpha)$ 
such that, $eval(C', aux_C, v_1)$ $eval(D', aux_C, v_2)$ are in  $S$ and  $v_2  \theta \alpha n$ does not hold.
By Lemma \ref{AS to models}, part (1), from the answer set $S$ we can construct  a $\varphi_n$-coherent model $I^*=(\Delta^*, \cdot^{I^*})$ of $K$ in which
$C^{I^*}(z_C )=\frac{v_1}{n}$ and $D^{I^*}(z_C )=\frac{v_2}{n}$, for domain element $z_C$.

From the existence of a $\varphi_n$-coherent model $I^*$ of $K$ it follows, by Proposition \ref{prop:existence_canonical_model}, that a canonical $\varphi_n$-coherent model $I=(\Delta, \cdot^I)$ of $K$ exists.
As $I$ is canonical, there must be an element $y \in \Delta$ such that $B^I(y) = B^{I^*}(z_C)$, for all concept names $B$ occurring in $K$.  
Therefore, $B^I(y)=\frac{v'}{n}$ iff $\mathit{eval(B',aux_C, v') \in S }$, for all concept names $B$ occurring in $K$.
In particular, $C^I(y)= \frac{v_1}{n}$ and $D^I(y)= \frac{v_2}{n}$.
Hence, there is a canonical $\varphi_n$-coherent model  of $K$ such that $D^I(y)= \frac{v_2}{n}$ and  $\frac{v_2}{n}  \theta \alpha$ does not hold.

To conclude that $I$  falsifies  $\tip(C) \sqsubseteq D \theta \alpha$,
we have still to prove that $y$ is $<_C$ minimal with respect to all domain elements in $\Delta$ in $I$, i.e., $y \in min_{<_C}(C_{>0}^I)$.
If $y$ were not in $min_{<_C}(C_{>0}^I)$, there would be a $z\in \Delta$ such that $z<_C y$, that is, $C^I(z)> C^I(y)$.
This leads to a contradiction. Assume $C^I(z)=\frac{v_3}{n} > \frac{v_1}{n}$, by Lemma \ref{AS to models}, part (2),  there is an answer set $S'$ of
$\Pi(K, n, C,D, \theta, \alpha)$  such that $eval(C', aux_C, v_3)$. However, this would contradict the hypothesis that $S$ is a preferred answer set of $\Pi(K, n, C,$ $D, \theta, \alpha)$,
as $v_3>v_1$.  \hfill $\Box$
\end{proof}

\medskip

\noindent
{\em Proposition \ref{prop:upper bound}} \\
$\varphi_n$-coherent entailment from a weighted  $G_n \lc \tip$ ($\L_n \lc \tip$) knowledge base is in  $\Pi^p_2$.

\begin{proof}
Let $K$ be a  weighted $G_n \lc \tip$ knowledge base $K$ (the proof for  $\L _n \lc \tip$ is similar).
We consider the complementary problem, that is, the problem of deciding whether $\tip(C) \sqsubseteq D \theta \alpha$ is not entailed by $K$ in the $\varphi_n$-coherent semantics.
It requires to determine whether there is a canonical  $\varphi_n$-coherent model of $K$ falsifying $\tip(C) \sqsubseteq D \theta \alpha$ or, equivalently (by Proposition  \ref{AS to pref-models}), whether there is a preferred answer set $S$ 
of $\Pi(K, n, C,D, \theta, \alpha)$ such that $eval(D',aux_C, v)$ belongs to $S$ and  $v \theta \alpha n$ does not hold.  

This problem can be solved by an algorithm that non-deterministically guesses a ground interpretation $S$ over the language of $\Pi(K, n, C,D, \theta, \alpha)$, 
of polynomial size (in the size of $\Pi(K, n, C,D, \theta, \alpha)$)
and, then, verifies that $S$ satisfies all rules in $\Pi(K, n, C,D, \theta, \alpha)$ and is supported in $S$ (i.e., it is an answer set of $\Pi(K, n, C,D, \theta, \alpha)$),  that $eval(D',aux_C, v)$ is in $S$, that  $v \theta \alpha n$ does not hold, and that 
$S$ is preferred among the answer sets of $\Pi(K,n,C,D,\theta, \alpha)$.
The last point can be verified using an NP-oracle which answers "yes" when $S$ is  a preferred answer set of $\Pi(K, n, C,D, \theta, \alpha)$, and "no" otherwise.

The oracle checks if there is an answer set $S'$ of $\Pi(K, n, C,D, \theta, \alpha)$ which is preferred to $S$, by
 non-deterministically guessing a ground polynomial interpretation $S'$ over the language of $\Pi(K, n, C,D, \theta, \alpha)$,  and verifying that $S$ satisfies all rules and is supported in $S'$ (i.e., $S'$ is an answer set of $\Pi(K,C,D,\theta \alpha)$),  and that $S'$ is preferred to $S$. These checks can be done in polynomial time. 

Hence, deciding whether $\tip(C) \sqsubseteq D \theta \alpha$ is not entailed by $K$ in the $\varphi_n$-coherent semantics is in $\Sigma^p_2$,  and the complementary problem of deciding $\varphi_n$-coherent  entailment is in $\Pi^p_2$. 
\hfill $\Box$
\end{proof}

\end{appendix}

\end{document}